\documentclass[11pt]{article}

\usepackage{fullpage}
\usepackage[utf8]{inputenc} %
\usepackage[T1]{fontenc}    %
\usepackage{hyperref}       %
\usepackage{url}            %
\usepackage{booktabs}       %
\usepackage{amsfonts}       %
\usepackage{nicefrac}       %
\usepackage{microtype}      %
\usepackage{amssymb,amsmath}
\usepackage{amsthm}
\usepackage{natbib}
\setcitestyle{open={(},close={)}}
\usepackage{multirow}
\usepackage{array,graphicx}
\usepackage[noend]{algpseudocode}
\usepackage{algorithm}
\usepackage{xr}
\usepackage{tikz}
\usetikzlibrary{positioning,matrix,calc,arrows,shapes,fit,decorations,decorations.pathreplacing}
\usepackage{stmaryrd}
\usepackage{booktabs} 
\usepackage{thmtools, thm-restate}
\usepackage{comment}

\newtheorem{theorem}{Theorem}
\newtheorem{lemma}{Lemma}

\newtheorem{proposition}{Proposition}
\newtheorem{corollary}{Corollary}
\newtheorem{remark}{Remark}
\newtheorem{definition}{Definition}

\newcommand{\calX}{\mathcal{X}}
\newcommand{\calY}{\mathcal{Y}}

\newcommand{\calD}{\mathcal{D}}
\newcommand{\calQ}{\mathcal{Q}}
\newcommand{\calL}{\mathcal{L}}

\renewcommand{\vec}[1]{\boldsymbol{#1}}

\newcommand{\bx}{\vec{x}}
\newcommand{\by}{\vec{y}}
\newcommand{\bz}{\vec{z}}

\newcommand{\bY}{\vec{Y}}
\newcommand{\bZ}{\vec{Z}}
\newcommand{\cby}{\dot{\by}}

\newcommand{\cbz}{\dot{\bz}}

\newcommand{\prob}{\mathbf{P}}

\newcommand{\heta}{\hat\eta}

\newcommand{\bobby}{\text{PLT training cost}}

\newcommand{\predcost}{c}

\newcommand{\R}{\mathbb{R}}
\DeclareMathOperator*{\argmin}{argmin}

\newcommand{\cD}{\mathcal{D}}
\newcommand{\cH}{\mathcal{H}}

\newcommand{\cT}{\mathcal{T}}
\newcommand{\Path}[1]{\mathrm{Path}(#1)}
\newcommand{\pa}[1]{\mathrm{pa}(#1)}

\newcommand{\childs}[1]{\mathrm{Ch}(#1)}
\newcommand{\ch}[1]{\mathrm{ch}(#1)}
\renewcommand{\deg}{\mathrm{deg}}
\newcommand{\lenpath}{\mathrm{len}}
\newcommand{\depth}{\mathrm{depth}}
\newcommand{\id}{\mathrm{id}}
\newcommand{\leaf}{\ell}

\newcommand{\ceil}[1]{\lceil #1 \rceil}

\newcommand{\assert}[1]{\mathbb{I}\left \{#1 \right \}}

\newcommand{\given}{\, | \,}

\newcommand{\Algo}[1]{\textsc{#1}}
\newcommand{\AlgoPLT}{\Algo{PLT}}

\title{On the computational complexity of the probabilistic label tree algorithms}

\begin{document}

\author{
R\'{o}bert Busa-Fekete\thanks{Yahoo! Research, email: busafekete@verizonmedia.com }
\and
Krzysztof Dembczy\'{n}ski\thanks{Institute of Computing Science,
Poznan University of Technology, email: kdembczynski@cs.put.poznan.pl}
\and
Alexander Golovnev\thanks{Harvard University, email: alexgolovnev@gmail.com. Supported by a Rabin Postdoctoral Fellowship.}
\and
Kalina Jasinska\thanks{Institute of Computing Science,
Poznan University of Technology, email: kjasinska@cs.put.poznan.pl}
\and 
Mikhail Kuznetsov\thanks{Yahoo! Research, email: kuznetsov@verizonmedia.com}
\and
Maxim Sviridenko\thanks{Yahoo! Research, email: sviri@verizonmedia.com }
\and
Chao Xu\thanks{Yahoo! Research, email: chao.xu@verizonmedia.com}
}
\date{}
\maketitle

\begin{abstract}
  \noindent
  Label tree-based algorithms are widely used to tackle multi-class and multi-label problems with a large number of labels. We focus on a particular subclass of these algorithms that use probabilistic classifiers in the tree nodes. Examples of such algorithms are hierarchical softmax (HSM), designed for multi-class classification, and probabilistic label trees (\AlgoPLT{s}) that  generalize HSM to multi-label problems. If the tree structure is given, learning of \AlgoPLT{} can be solved with provable regret guaranties~\citep{WydmuchJKBD18}. However, to find a tree structure that results in a \AlgoPLT{} with a low training and prediction computational costs as well as low statistical error seems to be a very challenging problem, not well-understood yet. 
  
  In this paper, we address the problem of finding a tree structure that has low computational cost. First, we show that finding a tree with optimal training cost is NP-complete, nevertheless there are some tractable special cases with either perfect approximation or exact solution that can be obtained in linear time in terms of the number of labels $m$. For the general case, we obtain $O(\log m)$ approximation in linear time too. Moreover, we prove an upper bound on the expected prediction cost expressed in terms of the expected training cost. We also show that under additional assumptions the prediction cost of a \AlgoPLT{} is $O(\log m)$.
\end{abstract}

\thispagestyle{empty}
\addtocounter{page}{-1}
\newpage    

\section{Introduction}

We consider a class of machine learning algorithms that use hierarchical structures of classifiers to reduce the computational complexity of training and prediction in large-scale problems characterized by a large number of labels. Problems of this type are often referred to as {extreme classification}~\citep{Prabhu_Varma_2014}. The hierarchical structure usually takes a form of a label tree in which a leaf corresponds to one and only one label. The nodes of the tree contain classifiers that direct the test examples from the root down to the leaf nodes. We study the subclass of these algorithms with probabilistic classifiers, i.e., classifiers with responses in the range $[0,1]$.
Examples of such algorithms for multi-class classification include hierarchical softmax (\Algo{HSM})~\citep{Morin_Bengio_2005}, as implemented for example in \Algo{fastText}~\citep{Joulin_et_al_2016}, and conditional probability estimation trees~\citep{Beygelzimer_et_al_2009b}. For multi-label classification this idea is known under the name of probabilistic label trees (\AlgoPLT{}s)~\citep{Jasinska_et_al_2016}, and has been implemented in \Algo{Parabel}~\citep{Prabhu_et_al_2018} and \Algo{extremeText}~\citep{WydmuchJKBD18}. Note that the \AlgoPLT{} model can be treated as a generalization of algorithms for both multi-class and multi-label classification~\citep{WydmuchJKBD18}.

We present a wide spectrum of theoretical results concerning training and prediction costs of \AlgoPLT{}s. We first define the multi-label problem (Section~\ref{sec:multi_label}). 
Then, we define the \AlgoPLT{} model and state some of its important properties (Section~\ref{sec:plt}). 
As a starting point of our analysis, we define the training cost for a single instance as the number of nodes where it is involved in training classifiers (Section \ref{sec:train_comp}). 
The rationale behind this cost is that the learning methods, often used to train the node classifiers, scale linearly with the sample size. 
We note that the popular 1-vs-All approach has the cost equal $m$, the number of labels, according to our definition. 
This cost can be significantly reduced by using \AlgoPLT{}s.
We then address the problem of finding a tree structure that minimizes the training cost (Section \ref{sec:optimizing}). 
We first show that the decision version of this problem is NP-complete (Section~\ref{sec:hardness}). 
Nevertheless, there exists a $O(\log m)$ approximation that can be computed in linear time (Section~\ref{sec:log-apx}). 
We also consider two special cases: 
multi-class (Section~\ref{sec:multi_class}) and multi-label with nested labels (Section~\ref{sec:nested}), 
for which we obtain constant approximation and exact solution, respectively, both computed in linear time in $m$. 
We also consider the prediction cost defined as the number of nodes visited during classification of a test example (Section~\ref{sec:prediction}). 
We first show that under additional assumptions prediction can be made in $O(\log m)$ time. 
Finally, we prove an upper bound on the expected prediction cost expressed in terms of the expected training cost and statistical error of the node classifiers.

The problem of optimizing the training cost is closely related to the binary merging problem in databases~\citep{7164931}. 
The hardness result in~\citep{7164931}, however, does not generalize to our setting as it is limited to binary trees only. 
Nevertheless, our approximation result is partly based on the results from~\citep{7164931}. The training cost we use is similar to the one considered in~\citep{Grave_et_al_2017}, but the authors there consider a specific class of shallow trees. 
The Huffman tree is a popular choice for \Algo{HSM} (many \Algo{word2vec} implementations~\citep{Mikolov_et_al_2013} and \Algo{fastText}~\citep{Joulin_et_al_2016} use binary Huffman trees). This strategy is justified as for multi-class with binary trees the Huffman code is optimal~\citep{WydmuchJKBD18}. Surprisingly, the solution for the general multi-class case has been unknown prior to this work.
The problem of learning the tree structure to improve the predictive performance is studied in~\citep{Jernite_et_al_2017, Prabhu_et_al_2018}. 
Ideally, however, one would like to have a procedure that minimizes two objectives: the computational cost and statistical error.

\section{Multi-label classification}
\label{sec:multi_label}

Let $\calX$ denote an instance space, and let $\calL = [m]$ be a finite set of $m$ class labels. 
We assume that an instance $\bx \in \calX$ is associated with a subset of
labels $\calL_{\bx} \subseteq \calL$ (the subset can be empty); this subset is often called the set of relevant labels, while the complement
$\calL \backslash \calL_{\bx}$ is considered as irrelevant for $\bx$. We assume $m$ to be a large number (e.g., $\ge 10^5$), but the size of the set of relevant labels $\calL_{\bx}$ is usually much smaller than $m$, i.e., $|\calL_{\bx}| \ll m$. We identify the set $\calL_{\bx}$ of relevant labels with the binary
vector $\by = (y_1,y_2, \ldots, y_m)$, in which $y_j = 1 \Leftrightarrow j \in \calL_{\bx}$. By $\calY = \{0, 1\}^m$ we denote the set of all possible label vectors.
We assume that observations $(\bx, \by)$ are generated independently and identically according to a 
probability distribution $\prob(\bx, \by)$ defined on $\calX \times \calY$. Observe that the above definitions include as special cases multi-class classification (where $\|\by\|_1=1$) and $k$-sparse multi-label classification (where $\|\by\|_1\le k$).\footnote{We use $[n]$ to denote the set of integers from $1$ to $n$, and $\|\bx\|_1$ to denote the $L_1$ norm of $x$.}

We are interested in multi-label classifiers that estimate conditional probabilities of labels, $\eta_j = \prob(y_j = 1 \vert \bx)$, $j \in \calL$, as accurately as possible, i.e., with possibly small $L_1$-estimation error, i.e., $|\eta_j(\bx) - \heta_j(\bx)|$, where $\heta_j(\bx)$ is an estimate of $\eta_j(\bx)$. This statement of the problem is justified by the fact that optimal predictions in terms of the statistical decision theory for many performance measures used in multi-label classification, such as the Hamming loss, precision@k, and the micro- and macro F-measure, are determined through the conditional probabilities of labels~\citep{Dembczynski_et_al_2010c,Kotlowski_Dembczynski_2015,Koyejo_et_al_2015}.

\section{Probabilistic label trees ({\AlgoPLT}s)}
\label{sec:plt}

We will work with the set $\cT$ of rooted, leaf-labeled trees with $m$ leaves. 
We denote a single tree by $T$ and its set of leaves by $L_T$. 
The leaf $\ell_j \in L_T$ corresponds to the label $j \in \calL$. 
The set of leaves of a (sub)tree rooted in an inner node $v$ is denoted by $L(v)$. The parent node of $v$ is denoted by
$\pa{v}$, and the set of child nodes by $\childs{v}$. 
The path from node $v$ to the root is denoted by $\Path{v}$. The length of the path, i.e., the number of nodes on the path, is denoted by $\lenpath_v$. 
The set of all nodes %
is denoted by $V_T$. 
The degree of a node $v \in V_T$, i.e., the number of its children, is denoted by $\deg_v=|\childs{v}|$.

\AlgoPLT{} uses tree $T$ to factorize the conditional probabilities of labels, $\eta_j(\bx) = \prob(y_j = 1 \vert \bx)$, for $j \in \calL$. 
To this end let us define for every $\by$ a corresponding vector $\bz$ of length $|V_T|$,\footnote{Note that $\bz$ depends on $T$, but $T$ will always be obvious from the context.} whose coordinates, indexed by $v \in V_T$,\footnote{We will also use leaves $v \in L_T$ to index the elements of vector $\by$.} are given by:  
$$
z_v = \assert{\textstyle \sum_{\ell_j \in L(v)} y_{j} \ge 1} \,, \quad \textrm{or equivalently by~} z_v = \textstyle \bigvee_{\leaf_j \in L(v)} y_{j} \, .
$$
With the above definition, it holds based on the chain rule that for any $v \in V_T$:
\begin{equation}
\eta_v(\bx) = \prob(z_v = 1 \given \bx) =  \prod_{v' \in \Path{v}} \eta(\bx, v') \,,
\label{eqn:plt_factorization}
\end{equation}
where $\eta(\bx, v) = \prob(z_v = 1 \vert z_{\pa{v}} = 1, \bx)$ for non-root nodes, and $\eta(\bx, v) = \prob(z_v = 1 \given \bx)$ for the root~(see, e.g.,~\citealt{Jasinska_et_al_2016}). 
Notice that for the leaf nodes we get the conditional probabilities of labels, i.e., 
\begin{equation}
\eta_{\leaf_j}(\bx) = \eta_j(\bx) \,, \quad \textrm{for~} \ell \in L_T \,.
\label{eqn:plt_leaf_prob}
\end{equation}

The following result states the relation between probabilities of the parent node and its children. 
\begin{proposition}
\label{prop:node_cond_prob_interval}
For any $T$ and $\prob(\by \vert \bx)$, the probability $\eta_v(\bx)$ of any internal node $v \in V_T \setminus L_T$ satisfies:
\begin{equation}
\max \left \{ \eta_{v'}(\bx): v' \in \childs{v}\right \} \le \eta_v(\bx) \le \min \left \{ 1, \textstyle \sum_{v' \in \childs{v}} \eta_{v'}(\bx) \right \}\,.
\label{eqn:pop_child_cond_prob}
\end{equation}
\end{proposition}
\begin{proof}
We first prove the first inequality. 
From the definition of tree $T$ and $z_v$, we have that $z_v = 1 \Rightarrow z_{\pa{v}} = 1$ since $L(v) \subset L(z_{\pa{v}})$. 
Taking the expectation with respect to $\prob(\by \vert \bx)$, we obtain that $\eta_{v'}(\bx) \le \eta_{v}(\bx)$ for every $v' \in \childs{v}$. 

For the second inequality, obviously we have  $\eta_{v}(\bx) \le 1$. 
Furthermore, if $z_{v} = 1$, then there exists at least one $v' \in \childs{v}$ for which $z_{v'} = 1$. 
In other words, $z_{v} \le \sum_{v' \in \childs{v}} z_{v'}$. 
Therefore, by taking expectation with respect to $\prob(\by \vert \bx)$ we obtain $\eta_v(\bx) \le \sum_{v' \in \childs{v}} \eta_{v'}(\bx)$. 
\end{proof}

To estimate $\eta(\bx, v)$, for $v \in V_T$, we use a function class $\cH : \R^d \mapsto [0,1]$ which %
contains probabilistic classifiers of choice, for example, logistic regressors. We assign a classifier from $\cH$ to each node of the tree $T$. We shall index this %
set of classifiers by the elements of $V_T$ as $H = \{ \heta(v) \in \cH : v \in V_T \}$. 
We also denote by $\heta(\bx, v)$ the estimate of $\eta(\bx, v)$ obtained for a given $\bx$ in node $v \in V_T$. 
The estimates obey the analogous equations to (\ref{eqn:plt_factorization}) and (\ref{eqn:plt_leaf_prob}). However, as the probabilistic classifiers $\heta(v) \in H$ can be trained independently from each other, Proposition~\ref{prop:node_cond_prob_interval} may not apply to the estimated probabilities. This can be fixed by a proper normalization during prediction. 

The quality of the estimates of conditional probabilities $\heta_j(\bx)$, $j \in \calL$ can be expressed in terms of the $L_1$-estimation error in each node classifier, i.e., by $\left | \eta(\bx,v) - \heta(\bx, v) \right |$. Based on similar results from \citep{Beygelzimer_et_al_2009a} and \citep{WydmuchJKBD18} we get the following bound, which for $\ell_j \in L_T$ gives the guarantees for $\heta_j(\bx)$, $j \in \calL$.  
\begin{theorem}
\label{thm:estimation_regret}
For any tree $T$ and $\prob(\by \vert \bx)$ the following holds for $v \in V_T$:
\begin{equation}
\left | \eta_v(\bx) - \heta_v(\bx) \right |  \leq  \sum_{v' \in \Path{v}} \eta_{\pa{v'}}(\bx) \left | \eta(\bx,v') - \heta(\bx, v')  \right | \,,
\label{eqn:estimation_bound_known}
\end{equation}
where for the root node $\eta_{\pa{r_T}}(\bx) = 1$.
\end{theorem}
\begin{proof}
This result can be found as a part of the proof of Theorem~1 in Appendix~A in~\citep{WydmuchJKBD18}. It is presented in Eq.~(6) therein. However, this result is stated only for conditional probabilities of labels $\eta_j(\bx)$ and their estimates $\heta_j(\bx)$. The generalization to any node $v \in V_T$ is straightforward as the chain rule~(\ref{eqn:plt_factorization}) applies to any node $v$ and the necessary transformations to get the result can be applied.
\end{proof}

\section{Training complexity}
\label{sec:train_comp}

Training data $\cD = \{ (\bx_{i},\by_{i})\}_{i=1}^n$ consist of tuples of feature vector $\bx_i \in \R^d$ and label vector $\by_i\in \{ 0,1\}^m$. The labels for the entire training set can be written in a matrix form $\bY = [y_{i,j}]$ whose $j$-th column is denoted by $\cby_j$. We also use a corresponding matrix $\bZ = [z_{i,v}]$, with columns indexed by $v \in V_T$ and denoted by $\cbz_v$. 

\begin{algorithm}
\caption{\Algo{PLT.AssignToNodes}$(T, \bx, \by)$}
\label{alg:plt-assign}
\begin{small}
\begin{algorithmic}[1]
\State $P = \emptyset$, $N = \{r_T\}$ \Comment{Initialize positive and negative nodes ($r_T$ added to deal with $\by$ of all zeros)}
\For{$j \in \calL_{\bx}$} \Comment{For all labels of the training example}
\State $v = \ell_j$  \Comment{Set $v$ to a leaf corresponding to label $j$}
\While{$v$ not null \textbf{and} $v \not \in P$} \Comment{On a path to the first positive node (excluded) or the root (included)}
\State $P = P \cup \{v\}$ \Comment{Assign a node to positive nodes} 
\State $N = N \setminus \{v\}$ \Comment{Remove the node from negative nodes if added there before} 
\For{$v' \in \childs{v}$} \Comment{For all its children}
\If{$v' \not \in P$} \Comment{If a child is not a positive node}
\State $N = N \cup \{v'\}$ \Comment{Assign it to negative nodes} 
\EndIf
\EndFor
\State $v = \pa{v}$ \Comment{Move up along the path} 
\EndWhile
\EndFor
\State \textbf{return} $(P,N)$ \Comment{Return a set of positive and negative nodes for the training example}
\end{algorithmic}
\end{small}
\end{algorithm} 
We define the training complexity of \Algo{PLT}s in terms of the number of nodes in which a training example $(\bx,\by)$ is used. This number follows from the definition of the tree and the \Algo{PLT} model~(\ref{eqn:plt_factorization}). We use each training example in the root (to estimate $\prob(z_{r_T} = 1\vert \bx)$) and in each node $v$ for which $z_{\pa{v}} = 1$ (to estimate $\prob(z_v = 1\vert z_{\pa{v}}, \bx)$).  
Therefore, we define the training cost for a single training example $(\bx, \by)$ by:
\begin{equation}
c(T, \by) = 1 + \sum_{v \in V_T \setminus r_T} z_{\pa{v}} \,.
\label{eqn:learning_cost}
\end{equation}
Algorithm~\ref{alg:plt-assign} shows the \textsc{AssignToNodes} method which identifies for a training example the set of \emph{positive} and \emph{negative nodes}, 
i.e., the nodes for which the training example is treated respectively as positive (i.e, $(\bx, z_v = 1)$) or negative (i.e., $(\bx, z_v = 0)$) (see the pseudocode and the comments there for details of the method).\footnote{Notice that the \textsc{AssignToNodes} method has time complexity $O(c(T, \by))$ assuming that the set operations are performed in time $O(1)$ (e.g., the set is implemented by hash table).} 
Based on this assignment a learning algorithm of choice, either batch or online, trains the node classifiers $\heta(v, \bx)$. 
The \emph{training cost} for set $\calD$ is then expressed by:
$$
c(T, \bY) =  \sum_{i=1}^n c(T,\by_i)\,.
$$

The above quantities are justified from the learning point of view by the following reasons. On the one hand, in an online setting, the complexity of an update of \AlgoPLT{} based on a single sample $(\bx, \by)$ is indeed $O(c(T,\by))$, using a linear classifier in the inner node trained by optimizing some smooth loss with stochastic gradient descent (\Algo{SGD}) which is often the method of choice along with \AlgoPLT{s}. Moreover, even if \Algo{SGD} is used in an offline setting, the SOTA packages, like \Algo{fastText}, run several epochs over the training data. Therefore, their training time is $O( c(T, \bY) \cdot \text{\#epochs} )$, not taking into account the complexity of other layers. On the other hand, if we update the inner node models in a batch setting, the training time is again linear in $c(T, \bY)$ for several large-scale learning methods whose training process is based on optimizing some smooth loss, such as logistic regression~\citep{Zhu16c}.

The next proposition gives an upper bound for the cost $c(T, \by)$.
\begin{proposition}
\label{prop:cost_upperbound}
For any tree $T$ and vector $\by$ it holds that:
$$
c(T, \by) \le 1 + \|\by\|_1 \cdot \depth_T\cdot \deg_T\,,
$$
where $\depth_T = \max_{v \in L_T} \lenpath_v - 1$ is the depth of the tree, and $\deg_T = \max_{v \in V_T} \deg_v$ is the highest degree of a node in $T$.
\end{proposition}
\begin{proof}
First notice that a training example is always used in the root node, either as a positive example $(\bx, 1)$, if $\|\by\|_1 > 0$, 
or as a negative example $(\bx, 0)$, if  $\|\by\|_1 = 0$. Therefore the cost is bounded by 1.
If $\|\by\|_1 > 0$, the training example is also used as a positive example in all the nodes on paths from the root to leaves corresponding to labels $j$ for which $y_j = 1$ in $\by$. 
As the root has been already counted, we have at most $\depth_T = \max_v \lenpath_v - 1$ such nodes for each positive label in $\by$. 
Moreover, the training example is used as a negative example in all siblings of the nodes on the paths determined above, 
unless it is already a positive example in the sibling node. 
The highest degree of node in the tree is $\deg_T$.
Taking the above into account, the cost $c(T, \by)$ is upperbounded by $1 + \|\by\|_1 \cdot \depth_T \cdot \deg_T$. 
The bound is tight, for example, if $\|\by\|_1 = 1$ and $T$ is a perfect $\deg_T$-ary tree 
(all non-leaf nodes have equal degree and the paths to the root from all leaves are of the same length).

\end{proof}
\begin{remark}
Consider $k$-sparse multi-label classification (i.e., $\|\by\|_1  \le k$). For a balanced tree of constant $\deg_T=\lambda (\ge 2)$ and $\depth_T=\log_{\lambda}{m}$, the training cost is $c(T, \by)=O(k\log m)$. 
\end{remark}
In the proposition below we express the cost in terms of vectors $\cbz_{v}$. Each such vector indicates the positive examples for node $v$.  We refer to $\|\cbz_{v}\|_1$ as the \emph{Hamming weight} of the node $v \in V_T$. Moreover, we use $c(v) = c(T,\bY,v)= \|\cbz_v\|_1\cdot\deg_v$ for the cost of the node $v\in V_T$.
\begin{proposition}
\label{prop:cost_decomp_training_set}
For any tree $T$ and label matrix $\bY$ it holds that:
$$
c(T, \bY) =  n + \!\!\!\sum_{v \in V_T \setminus r_T} \|\cbz_{\pa{v}}\|_1 = n + \!\!\!\sum_{v \in V_T}  \|\cbz_{v}\|_1\cdot \deg_v=n + \!\!\!\sum_{v \in V_T}  c(v)\,.
$$
\end{proposition}
\begin{proof}
Obviously, we have that:
$$
\sum_{i=1}^n c(T,\by) = \sum_{i=1}^n \left ( 1 + \sum_{v \in V_T\setminus r_T} z_{i,\pa{v}} \right ) = n + \sum_{v \in V\setminus r_T} \|\cbz_\pa{v}\|_1 \,,
$$
as elements $z_{i,\pa{v}}$ constitute matrix $\bZ = [\cbz_1, \ldots, \cbz_{|V|}]$ with columns $\cbz_v$ corresponding to the nodes of $T$.
Next, notice that for each $v \in V_T \setminus L_T$, we have:
$$
\sum_{v' \in \childs{v}} z_{v} = z_{v} \sum_{v' \in \childs{v}} 1 = z_{v} \cdot \deg_v \,.
$$
Therefore, 
$$
\sum_{i=1}^n c(T,\by) = n + \sum_{v \in V_T\setminus r_T} \|\cbz_\pa{v}\|_1 = n + \sum_{v \in V} \|\cbz_v\|_1 \cdot \deg_v \,.
$$
The last sum is over all nodes as for $v \in L_T$ we have $\deg_v = 0$. The final equation is obtained by definition of the cost of the node $v \in V_t$, i.e., $c(v) = c(T,\bY,v)= \|\cbz_v\|_1\cdot\deg_v$.
\end{proof}
Next we show a counterpart of Proposition~\ref{prop:node_cond_prob_interval} for training data. 
\begin{proposition}
\label{prop:node_Hamming_weight_interval}
For any $T$ and label matrix $\bY$, the Hamming weight $\|\cbz_{v}\|_1$ of any internal node $v \in V_T \setminus L_T$ satisfies:
\begin{equation}
\max \left \{ \|\cbz_{v'}\|_1: v' \in \childs{v}\right \} \le \|\cbz_{v}\|_1 \le \min \left \{ n, \textstyle \sum_{v' \in \childs{v}} \|\cbz_{v'}\|_1 \right \}\,,
\label{eqn:training_child_cond_prob}
\end{equation}
with equality on the left holding for label covering distributions, i.e., $\forall \by_i \exists \ell_j \in L(v): \forall_{\ell_k \in L(v)\setminus\ell_j} (y_{i,k}\!=\!1\!\Rightarrow\!y_{i,j}\!=\!1)$, and equality on the right holding for multi-class distributions, i.e., $\forall \by_i \sum_{\ell_j \in L(v)} y_{i,j} = 1$. 
\end{proposition}
\begin{proof}
The proof follows the same steps as the proof of Proposition~\ref{prop:node_cond_prob_interval} with the difference that instead of expectation with respect to $\prob(\by \given \bx)$, we take the sum over the training examples. 

The left inequality becomes equality, for example, for the label covering distribution, since $z_{v} = z_{v'}$ for the child node $v'$ under which there is label $j$, i.e., $j \in L(v')$, or $v'$ is the leaf node corresponding to label $\ell_j$.

The right inequality becomes equality, for example, for the multi-class distribution, since there is always only one child $v'$ for which $z_{v'} = 1$.
\end{proof}

Another important quantity we use is the expected training cost:
$$
C_\prob(T) = \mathbb{E}_{\by} \left [ c(T,\by) \right ] = \sum_{\by \in \calY} c(T,\by) \prob(\by) \,. 
$$
Propositions \ref{prop:cost_upperbound}--\ref{prop:node_Hamming_weight_interval} can be easily generalized to the expected training cost.
\begin{proposition}
\label{prop:cost_decomp_basic}
For any tree $T$ and distribution $\prob(\by)$ it holds that:
$$
C_\prob(T)  = 1 + \sum_{v \in V_T \setminus r_T} \prob(z_{\pa{v}} = 1) =  1 + \sum_{v \in V_T} \prob(z_{v} = 1) \cdot \deg_v \,.
$$
\end{proposition}
\begin{proof}
The result follows immediately by taking the expectation of $c(T, \by)$ and the same observation as in Proposition~\ref{prop:node_Hamming_weight_interval}. For $v \in V_T \setminus L_T$, we have:
$$
\sum_{v' \in \childs{v}} z_{v} = z_{v} \sum_{v' \in \childs{v}} 1 = z_{v} \cdot deg_v \,,
$$
Namely, we have
\begin{eqnarray*}
C_\prob(T) = \mathbb{E} [ c(T,\by) ] & = & \sum_{\by} c(T,\by) \prob(\by) \\
& = & \sum_{\by} \left ( 1 + \sum_{v \in V_T\setminus r_T} z_{\pa{v}} \right ) \prob(\by) \\
& = & 1 + \sum_{v \in V_T \setminus r_T} \sum_{\by} z_{\pa{v}} \prob(\by) \\
& = & 1 + \sum_{v \in V_T \setminus r_T} \prob(z_{\pa{v}} = 1) 
\\ & = & 1 + \sum_{v \in V_T} \prob(z_{v} = 1) \cdot \deg_v \,.
\end{eqnarray*}
The last sum is over all nodes as for $v \in L_T$ we have $\deg_v = 0$. 
\end{proof}

\begin{proposition}
\label{prop:pop_cost_upperbound}

For any tree $T$ and distribution $\prob(\by)$ it holds that:
$$
C_{\prob}(T) \le  1 + \depth_T \cdot \deg_T \cdot \sum_{j=1}^m \prob(y_j = 1) \,,
$$
where $\depth_T = \max_{v \in L_T} \lenpath_v - 1$ is the depth of the tree, and $\deg_T = \max_{v \in V_T} \deg_v$ is the highest degree of a node in $T$. 
\end{proposition}
\begin{proof}
The proof follows immediately from Proposition~\ref{prop:cost_upperbound} by taking the expectation over $\prob(\by)$.
\end{proof}

\begin{proposition}
\label{prop:node_prob_interval}
For any $T$ and distribution $\prob$, the probability $\prob(z_v = 1)$ of any internal node $v \in V(T)\setminus L(T)$ satisfies:
\begin{equation}
\max \left \{ \prob(z_{v'} = 1): v' \in \childs{v}\right \} \le \prob(z_v = 1) \le \min \left \{ 1, \sum_{v' \in \childs{v}} \prob(z_{v'} = 1) \right \}\,.
\label{eqn:pop_child_prob}
\end{equation}
\end{proposition}
\begin{proof}
The proposition follows immediately from Proposition~\ref{prop:node_cond_prob_interval} by taking the expectation over $\prob(\bx)$. 
\end{proof}

Next, we state the relation between the finite sample and expected training costs. 
Using the fact that $c(T, \by)$ has bounded difference property, we can compute its deviation from its mean as follows.
\begin{proposition}
\label{prop:conc_train_cost}
For any PLT with label tree $T$, it holds that
\[
\prob \big ( \left| c(T, \by ) -  C_\prob(T) \right|>\epsilon \big) \le 2e^{-2\epsilon^2/\sum_{i=1}^m d_i^2}\,,
\]
where $d_i=\sum_{j\in \Path{\ell_i} } \deg_{\pa{j}}$.
\end{proposition}
\begin{proof}
We can directly apply the concentration result for functions with bounded difference (see Section 3.2 of \citealt{BoLuMa13}). It only remains to upper bound $\sup_{\by, \by^{(i)}} | c(T, \by ) - c(T, \by^{(i)}) | $ for any $i$, where $\by^{(i)}$ is the same as $\by \in \{ 0,1\}^m$ except that the $i$ component is flipped. First, consider the case when $y_i=0$ and let us flip its value. Based on Proposition~\ref{prop:cost_decomp_basic}, the training algorithm of \Algo{PLT} %
updates each children of an inner node $v$ if there is at least one leaf $\ell$ in the subtree below it for which $y_\ell=1$, otherwise it does not update the children classifier with the given example. Thus $| c(T, \by ) - c(T, \by^{(i)}) |$ cannot be bigger than $d_i$. The same argument applies to the case when $y_i = 1$ which concludes the proof.
\end{proof}
Note that $d_i \le 2 \log m$ for balanced binary trees, thus $\tfrac{1}{n} \sum_{i=1}^n c(T, \by_i)$ is close to its expected value with $\Omega (\sqrt{m} \log m)$ samples with high probability. This lower bound suggests that one should not consider optimizing the training complexity based on fewer examples, since the empirical value $\tfrac{1}{n} \sum_{i=1}^n c(T, \by_i)$ which one would like to optimize over the space of labeled trees, might significantly deviate from its expected value.

\section{Optimizing the training complexity \texorpdfstring{($\min_{T\in \cT} c(T,\bY )$)}{}}
\label{sec:optimizing}

In this section, we focus on the algorithmic and hardness results for minimizing the cost $c(T,\bY)$. In the analysis, we mainly refer to matrices $\bY$ and $\bZ$ via their columns $\cby_j \in \{0,1\}^n$, $j \in [m]$, and $\cbz_v \in \{0,1\}^n$, $v \in V_T$, respectively. We assume $\bY$ to be stored efficiently, for example, as a sparse matrix whenever it is possible. We also use $p_j=\|\cby_j\|_1/n$ and $p_v=\|\cbz_v\|_1/n$, which are the fractions of positive examples in the corresponding nodes. 

\subsection{Hardness of training cost minimization}
\label{sec:hardness}
First we formally define the decision version of the cost minimization problem. 
\begin{definition}[\bobby{} problem]
For a label matrix $\bY$ and a parameter $w$ decide whether there exists a tree $T \in \cT$ such that $c(T,\bY) \le w$.
\end{definition}

We prove NP-hardness of PLT training cost by a reduction from the Clique problem (which is one of the classical NP-complete problems~\citealt{GJ79}) defined as follows. 
\begin{definition}[Clique]
For an undirected graph $G=(V,E)$ and a parameter $1\leq k\leq |V|$, decide whether $G$ contains a clique on $k$ nodes.
\end{definition}
\begin{restatable}{theorem}{npcomplete}
\label{thm:np_complete}
The \bobby{} problem is NP-complete.
\end{restatable}
We remark that a problem similar to \bobby{} has been studied in the database literature. In particular, the problem of finding an optimal \emph{binary} tree is proven to be NP-hard in~\citep{7164931}. Note that the result of~\citep{7164931} does not imply hardness of the \bobby{} problem.

\subsection{Logarithmic approximation for multi-label case}
\label{sec:log-apx}
Despite the hardness of the problem, we are able to give a simple algorithm which achieves an $O(\log m)$ approximation.
As remarked above, the problem of finding an optimal \emph{binary} PLT tree is equivalent to the binary merging problem considered in~\citep{7164931}.

\begin{definition}[Binary merging]
For a ground set $U$ of size $n$, and a collection of $m$ sets,~$A_1,...,A_m$ where each $A_i\subseteq U$, a merge schedule is a pair of a full binary tree\footnote{A full binary tree is a tree where every non-leaf node has exactly $2$ children.} $T_{binmerge}^*$ with $m$ labeled leaves, and a permutation~$\pi:[m]\rightarrow [m]$ which assigns every set~$A_i$ to the leaf number~$\pi(i)$. The binary merging problem is to find a merge schedule of the minimum cost:
\[
\text{cost}(T,\pi,A_1,...,A_m)=\sum\limits_{v\in T}|A_v| \, ,
\]
where~$A_v$ is the union of sets~$A_i$ assigned to the leaves of the subtree rooted at the node~$v$.     
\end{definition}

While binary merging is NP-complete, it admits an $O(\log m)$ approximation~\citep{7164931}. 
The lemma below, showing that any \bobby~problem can be 2-approximated by a binary PLT tree, gives a simple $O(\log{m})$-approximation for the \bobby{} problem: it suffices to find an optimal binary tree using the algorithms from~\citep{7164931} (e.g., one of the algorithms presented there is a simple modification of the Huffman tree building algorithm). 

\begin{lemma}
\label{lem:binary}
For any \bobby~instance $\bY$, it holds that
\[
\min_{T\in \cT} c(T, \bY) \le 2 \min_{T\in \cT_{\text{bin}}} c(T, \bY)
\, ,
\]
where $\cT_{\text{bin}}$ denotes the set of trees in which each internal node (including the root) has degree $2$.
\end{lemma}
\begin{proof}
Consider an optimal tree $T^*_{\bY} \in \argmin_{T\in \cT} c(T,\bY)$. Starting from the root, replace every node with more than $2$ children by an arbitrary binary tree whose set of leaves is the set of children of this node. Consider a node $v$ of $T^*_{\bY}$, let $v_1,\ldots,v_d$ be the children of $v$. The cost of the node $v$ is $c(v)=\deg_v\cdot\|\cbz_{v_1}\vee\ldots\vee \cbz_{v_d}\|_1$. Any binary tree with the leaves $v_1,\ldots,v_d$ has $(\deg_v-1)$ internal nodes, each of them has degree two and the Hamming weight of its label $\cbz$ is at most $\|\cbz_{v_1}\vee\ldots\vee \cbz_{v_d}\|_1$. Thus, the sum of the costs of the internal nodes of this binary tree is at most $2(\deg_v-1)\cdot\|\cbz_{v_1}\vee\ldots\vee \cbz_{v_d}\|_1 < 2\deg_v\cdot\|\cbz_{v_1}\vee\ldots\vee \cbz_{v_d}\|_1=2c(v)$. When we repeat this procedure for all internal nodes of $T^*_{\bY}$, we increase the cost of each node by at most a factor of $2$. Thus, the resulting binary tree is a $2$-approximation of $T^*_{\bY}$.
\end{proof}
\begin{figure}[ht]
\begin{center}
\scalebox{1}{
\begin{tikzpicture}
[every node/.style={draw=black,rectangle,minimum width=8mm,minimum height=5mm,black,inner sep=.5mm,rounded corners=1ex,draw,font=\footnotesize,inner sep=3pt},
level 1/.style={sibling distance=14mm},
level 2/.style={sibling distance=14mm}, 
level distance=9mm,line width=.3mm,
scale=1,transform shape]


\node at (5,5.3) {$v$} [grow=down]
  child {node {$\cbz_{v_1}$}}
  child {node {$\cbz_{v_2}$}}
  child {node {$\cbz_{v_3}$}}
  child {node {$\cbz_{v_4}$}}
  ;
  
\node at (12,6) {$v$} [grow=down]
  child {node {$\cbz_{v_1}$}}
  child {node {$t_1$}
  child {node {$\cbz_{v_2}$}}
  child {node {$t_2$}  
  child {node {$\cbz_{v_3}$}}
  child {node {$\cbz_{v_4}$}}}}
  ;

\draw[->] (8,5) -- (10, 5);
\node[draw=none] at (5,5.8) {cost$=4\|\cbz_{v_1}\vee \cbz_{v_2}\vee\cbz_{v_3}\vee \cbz_{v_4}\|_1$};  
\node[draw=none] at (12,6.5) {cost$=2\|\cbz_{v_1}\vee \cbz_{v_2}\vee\cbz_{v_3}\vee \cbz_{v_4}\|_1$};  
\node[draw=none] at (15.2,5.1) {cost$=2\|\cbz_{v_2}\vee\cbz_{v_3}\vee \cbz_{v_4}\|_1$}; 
\node[draw=none] at (15.5,4.2) {cost$=2\|\cbz_{v_3}\vee \cbz_{v_4}\|_1$}; 
\end{tikzpicture}
}
\caption{The transformation from Lemma~\ref{lem:binary}: the node $v$ of high degree is transformed into a binary tree of total cost less than twice the original cost: $2\|\cbz_{v_1}\vee \cbz_{v_2}\vee\cbz_{v_3}\vee \cbz_{v_4}\|_1 + 2\|\cbz_{v_2}\vee\cbz_{v_3}\vee \cbz_{v_4}\|_1+2\|\cbz_{v_3}\vee \cbz_{v_4}\|_1\leq 6\|\cbz_{v_1}\vee \cbz_{v_2}\vee\cbz_{v_3}\vee \cbz_{v_4}\|_1\leq 2 \cdot 4\|\cbz_{v_1}\vee \cbz_{v_2}\vee\cbz_{v_3}\vee \cbz_{v_4}\|_1$.}
\end{center}
\end{figure}
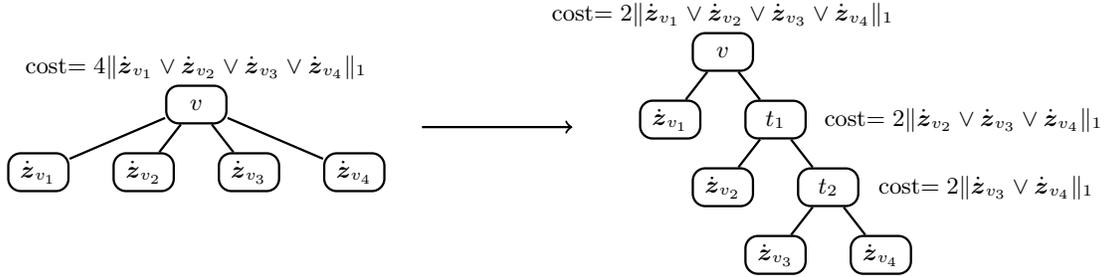
We are able, however, to give another algorithm, based on ternary complete trees, with a slightly better constant in the approximation ratio.\footnote{We use $\log$ to denote the logarithm base $2$.} 
\begin{theorem}
\label{thm:binaryApprox}
There exists an algorithm which runs in time $O(m+n)$ and achieves an approximation guarantee of $\frac{3\log{m}}{\log{3}}$ for the \bobby{} problem, i.e., the output $T$ of the algorithm satisfies
\[
c(T, \bY ) \leq \frac{3\log{m}}{\log{3}}\cdot \min_{T\in \cT} c(T, \bY ) \, .
\]
\end{theorem}
\begin{proof}
The algorithm constructs in linear time a complete ternary tree $T$ of depth $\ceil{\log_3{m}}$, and assigns the $m$ vectors $\cby_i$ to the leaves arbitrarily. From the definition of the cost function we have that for every tree $T^*\colon c(T^*, \bY)\geq n+\sum_{i=1}^m\|\cby_i\|_1$. On the other hand, from Proposition~\ref{prop:cost_decomp_training_set} we have that $c(T, \bY)\leq n+3\ceil{\log_3{m}}\sum_{i=1}^m \|\cby_i\|_1$, which completes the proof.
\end{proof}
We remark that any improvement of the approximation ratio of Theorem~\ref{thm:binaryApprox} would solve an open problem.
Indeed, since the proof of Lemma~\ref{lem:binary} is constructive and efficient, any $o(\log{m})$-approximation algorithm for the \bobby{} problem would imply an $o(\log{m})$-approximation of an optimal binary tree, and this would improve the best known approximation ratio for the binary merging problem.

\subsection{Multi-class case}
\label{sec:multi_class}

In the multi-class case, we have $\sum_{j=1}^m y_{i,j} = 1$ for each $\by_i$ in $\bY$. 
For ease of exposition, we assume that the columns $\cby_1,\ldots,\cby_m$ are sorted such that $0<p_1\leq\ldots\leq p_m$. 

Remark that for trees of a fixed degree $\lambda$ for all internal nodes, the optimal solution is the $\lambda$-ary Huffman tree. 
Here, we do not have this restriction and have different costs for nodes of different degrees, which makes the problem more difficult. 
Nevertheless, we give two efficient algorithms which find almost optimal solutions for every instance of the multi-class \bobby~problem. 
Namely, these algorithms find a solution within a small additive error. 
Moreover, these algorithms run in linear time $O(n+m)$. 

We will use the entropy function defined as $H(p_1,\ldots,p_k) = \sum_{i=1}^{k}p_j\log (1/p_j)$, for $k\geq1$, and $p_1,\ldots,p_k>0$.\footnote{For ease of exposition, we do not require the arguments of the entropy function to sum up to $1$.}
We will use the fact that for $p_1+\ldots+p_k\leq 1, H(p_1,\ldots,p_k)\leq \log{k}$ (this follows from Jensen's inequality).
We will also make use of the following corollary of Jensen's inequality.
\begin{proposition}
\label{prop:entropy}
Let $k\geq1$, and $p_1,\ldots,p_k>0$. Let $p=\sum_{i=1}^k p_i$. Then
\begin{align*}
H(p) \geq H(p_1,\ldots,p_k)-p\log{k} \; .
\end{align*}
\end{proposition}
\begin{proof}
Since $x\log\left(\frac{1}{x}\right)$ is concave for $x>0$, by Jensen's inequality we have that:
\begin{align*}
H(p_1,\ldots,p_k)&=\sum_{i=1}^{k}p_i\log\left(\frac{1}{p_i}\right) \\
&\leq k\cdot\left(\frac{p}{k}\right)\log\left(\frac{k}{p} \right)\\
&=p\left(\log\left(\frac{1}{p}\right)+\log{k}\right)\\
&=H\left(p\right)+p\log{k} \; .
\end{align*}
\end{proof}

We start by showing a lower bound for the multi-class case.
\begin{lemma}
\label{lemma:lower_bound}
Let $\bY$ be an instance of the multi-class case. The cost of any tree $T$ for $\bY$ is at least
\begin{align*}
c( T, \bY) \geq n+\frac{3n}{\log{3}}\cdot H(p_1,\ldots,p_m) \; .
\end{align*}
\end{lemma}
\begin{proof}%
We prove this Lemma by induction on the number of inner nodes of $T$. If $T$ has only one inner node (the root), then 
\begin{align*}
c(T,\bY)=n+nm\geq n+n\cdot \frac{3\log{m}}{\log{3}}\geq n+\frac{3n}{\log{3}}\cdot H(p_1,\ldots,p_m) \; ,
\end{align*}
because $\log{m}\geq H(p_1,\ldots,p_m)$ for every integer $m\geq 1$.

Now assume $T$ has more than one inner nodes.
Consider an inner node $v$ of $T$ on the longest distance from the root. All children of $v$ are leaves. W.l.o.g. assume that the children of $v$ are $\cby_1,\ldots,\cby_k$ for $k\geq 2$. In the multi-class case we have that $\|\cbz_v\|_1=n \cdot \sum_{i=1}^k p_i$, and the cost $c(v)=\deg_v \cdot n \cdot \sum_{i=1}^k p_i = kn \cdot \sum_{i=1}^k p_i$. Now let $T'$ be the tree $T$ with the children of $v$ removed (while keeping the label $\cbz_v$ of the new leaf $v$). Then $c(T, \bY)=kn \cdot \sum_{i=1}^k p_i + c(T',\bY')$ where $\bY'$ derived from $\bY$ by replacing the columns $\cby_1,\ldots,\cby_k$ by the column $\cby_1\vee\ldots\vee\cby_k$. By the induction hypothesis, $c(T', \bY')\geq n+\frac{3n}{\log{3}}\cdot H(\sum_{i=1}^k p_i,p_{k+1},\ldots,p_m)$. Let $p=\sum_{i=1}^k p_i$. Then we have that
\begin{align*}
c(T, \bY)&=kn p + c(T', \bY')\\
&\geq n+knp + \frac{3n}{\log{3}}\cdot H\left(p,p_{k+1},\ldots,p_m\right) \\
&= n+knp + \frac{3n}{\log{3}}\cdot p\log\left(\frac{1}{p}\right)+\frac{3n}{\log{3}}\cdot H(p_{k+1},\ldots,p_m) \\
&\geq n+knp + \frac{3n}{\log{3}}\cdot (H(p_1,\ldots,p_k)-p\log{k})+\frac{3n}{\log{3}}\cdot H(p_{k+1},\ldots,p_m) \\
&=n+knp -\frac{3np\log{k}}{\log{3}}+\frac{3n}{\log{3}}\cdot H(p_{1},\ldots,p_m)\\
&=n+np\left(k -\frac{3\log{k}}{\log{3}}\right)+\frac{3n}{\log{3}}\cdot H(p_{1},\ldots,p_m)\\
& \geq n+\frac{3n}{\log{3}}\cdot H(p_{1},\ldots,p_m) \, ,
\end{align*}
where the second inequality is due to Proposition~\ref{prop:entropy}, and the last ineqaulity $k -\frac{3\log{k}}{\log{3}}\geq0$ holds for every integer $k\geq1$.
\end{proof}
As an upper bound, we prove that both a ternary Shannon code and a ternary Huffman tree give an almost optimal solution in the multi-class case. Both algorithms will construct a tree $T$ where each node (possibly except for one) has exactly three children. Remark that in the multi-class case the Hamming weight of each internal node is the sum of the Hamming weights of all leaves in its subtree (which follows from Proposition~\ref{prop:node_Hamming_weight_interval}). %

\begin{theorem}
\label{thm:multi-class_upper_bound}
A ternary Shannon code and a ternary Huffman tree for $p_1\leq\ldots\leq p_m$, which both can be constructed in time $O(n+m)$, solve the multi-class \bobby{} problem with an additive error of at most $3n$, i.e., the output $T$ of the algorithm satisfies
\[
c(T, \bY ) \leq \min_{T\in \cT} c(T, \bY ) + 3n \, .
\]
\end{theorem}
\begin{proof}%
Recall that for a leaf $i$ corresponding to the vector $\cby_i$, $\lenpath_i$ denotes the number of nodes on the path from $i$ to the root of the tree. Since in ternary Shannon and Huffman trees, the degree of each node is at most $3$, the total cost of these trees is at most $c(T, \bY )\leq n+ 3n\cdot\sum_{i=1}^m  (\lenpath_i-1) p_i$.

It is known that the value of the Shannon code is upper bounded by $v=\sum_{i=1}^m(\lenpath_i-1)p_i<\sum_{i=1}^{m}p_i\log_3\left(\frac{1}{p_i}\right)+1$ (see, e.g., Section~5.4 in~\citealt{CT12}). This implies that the cost of the corresponding ternary Shannon tree $T$ is
\begin{align*}
c(T,\bY)&\leq n+3n\cdot\sum_{i=1}^m (\lenpath_i-1) p_i\\
&< n+3n\left(\sum_{i=1}^{m}p_i\log_3\left(\frac{1}{p_i}\right)+1\right)\\
&=n+\frac{3n}{\log{3}}\cdot H(p_1,\ldots,p_m)+3n \; .
\end{align*}

It is also know that the weight of the ternary Huffman code is upper bounded by the same quantity $\sum_{i=1}^{m}p_i\log_3\left(\frac{1}{p_i}\right)+1$ (see, e.g., Section~5.8 in~\citealt{CT12}). Thus, the same upper bound holds for a ternary Huffman tree for the \bobby{} problem. This, together with Lemma~\ref{lemma:lower_bound}, implies approximation with an additive error of at most $3n$.

Now we show that in our case, PLT trees corresponding to Shannon and Huffman codes can be constructed even more efficiently. We assume a sparse representation of th input by the numbers $n_i=\|\cby_i\|_1$. From now on we will only store and work with $n_i$. Since all $n_i$ are integers from $1$ to $n$, we can sort them using Bucket sort in time $O(n+m)$. In Shannon code, the depth $| \Path{\ell_i} |=\ceil{\log_3(1/p_i)}$. We can construct the corresponding tree $T$ going from the root. We add internal nodes one by one, and connect leaves of the corresponding depth to this tree in the ascending order of $n_i$. This algorithm takes one pass over the sorted data, and also runs in time $O(n+m)$. Thus, the running time of the algorithm is $O(n+m)$.

For the Huffman code, we will also store a Bucket sorting of the current set of $n_i$. Namely, we introduce an array $s[1\ldots n]$ where $s[i]$ equals the number of vectors of Hamming weight $i$. Initially, this array can be constructed in time $O(n+m)$. In each iteration of the Huffman algorithm, we choose three smallest elements, and add a new one with a larger value of $n_i$. To implement all iterations of this procedure, it suffices to make only one pass through the array $s$ from the index $1$ to the index $n$. The running time of this algorithm is then again $O(n+m)$. This concludes the proof.
\end{proof}
This approximation is quite tight for the multi-class case, since it implies that $\tfrac{1}{n} \sum_{i=1}^n c(T, \by_i ) \leq \tfrac{1}{n} \sum_{i=1}^n c(T^*_{\bY}, \by_i) + 3$ which means that the difference with respect to the optimal tree is at most $3$ on average. Also, note that any algorithm for the multi-class case trivially gives an approximation for the $k$-sparse multi-label case. For example, the algorithm from Theorem~\ref{thm:multi-class_upper_bound} finds a solution for the $k$-sparse multi-label case of cost at most $k\cdot(\min_{T\in \cT} c(T, \bY ) + 3n)$.

Below we show that for the multi-class case there exists an optimal tree $T^*\in \cT$ which has a form similar to the tree from Theorem~\ref{thm:multi-class_upper_bound}: every internal node of $T^*$ also has $2$ or $3$ children nodes, and the order of the leaves in $T^*$ coincides with the order of the leaves in Theorem~\ref{thm:multi-class_upper_bound}.
\begin{lemma}
\label{lem:aux_1}
For the multi-class case, there exists an optimal tree $T^*_{\bY} \in \argmin_{T\in \cT} c(T, \bY )$ in which each internal node has $2$ or $3$ children. Moreover, in this tree the order of the leaves in descending order of their depths is the ascending order of their Hamming weights.
\end{lemma}
\begin{proof}%
Consider an optimal PLT tree~$T^*$. Consider a node~$v\in T^*$ having children ~$v_1,...,v_k,~k\geq 4$ (if no such node exist, $T^*$ already has the desired structure). Assume that $\|\cbz_{v_1}\|_1\leq\ldots\leq \|\cbz_{v_k}\|_1$. Consider the subtree~$T(v)$ rooted at the node~$v$: its root node has the cost~$c(v)=k\sum\limits_{i=1}^k \|\cbz_{v_i}\|_1$, and the whole subtree has the cost~$c(T(v), \bY)=k\sum\limits_{i=1}^k \|\cbz_{v_i}\|_1+\sum\limits_{i=1}^k c(T(v_i),\bY)$ where~$c(T(v_i),\bY)$ is the cost of the subtree rooted at~$v_i$.

We make the following transformation of the subtree rooted at~$v$: move the children nodes~$v_1,...,v_{k-2}$ under a new node~$u$, and make~$u$ a new child of~$v$ such that $v$ has three child nodes ($v_{k-1},v_{k-2},u$), and $u$ has $k-2$ child nodes ($v_1,...,v_{k-2}$). The cost of the new subtree rooted at $v$ is~$c(T'(v),\bY)=3\sum\limits_{i=1}^k \|\cbz_{v_i}\|_1+(k-2)\sum\limits_{i=1}^{k-2}\|\cbz_{v_i}\|_1+\sum\limits_{i=1}^k c(T(v_i),\bY)$.

Let $S=\sum\limits_{i=1}^k \|\cbz_{v_i}\|_1$. Then we have that $\sum\limits_{i=1}^{k-2}\|\cbz_{v_i}\|_1\leq \frac{k-2}{k}\cdot S$.
Now consider the difference of the costs of the transformed and original subtrees:
\begin{align*}
c(T'(v),\bY)-c(T(v),\bY)&=(3-k)\sum\limits_{i=1}^k \|\cbz_{v_i}\|_1+(k-2)\sum\limits_{i=1}^{k-2}\|\cbz_{v_i}\|_1\\
&\leq (3-k)S+\frac{(k-2)^2}{k}S \\
&= S\left(\frac{4}{k}-1\right)\, .
\end{align*}
Thus, the cost of the trasformed tree never exceeds the cost of the original tree for $k\geq 4$. We will apply this transformation to every node of degree~$\geq 4$. Since each such transformation decreases the total number of children of nodes with more that $3$ children, this process will eventually terminate. Thus, after a finite number of steps, we will get a tree where each node has degree $2$ or $3$.  

To prove the second part of the lemma, we observe that if the ascending Hamming weight order of the leaves does not match the  descending order of their depth, then we can swap two leaves without increasing the cost of the tree. Repeating this transformation a finite number of times we get the desired ordering of the leaves.
\end{proof}
Although Lemma~\ref{lem:aux_1} gives a characterization of the shape and order of an optimal tree, the number of trees of this form is still exponential in $m$.

\subsection{Nested multi-label case (Matryoshka label structure)}
\label{sec:nested}

In this section, we study the case where the labels have nested structure which is also known as Matryoshka structure. For $\cby,\cby'\in\{0,1\}^n$, we say $\cby\leq \cby'$ if $\forall i\in[n]\colon y_i\leq y'_i$. In this section we will assume that the $m$ vectors of $\bY$ satisfy $\cby_1\leq\ldots\leq \cby_m$. We also assume that $\cby_1$ contains at least one positive element: $\|\cby_1\|_1\geq1$. We start with two structural results for an optimal tree in this case.
\begin{lemma}
\label{lem:innerdegree}
Let $\bY$ be an instance of the nested multi-label case. There exists an optimal tree $T_{\bY}^* \in \argmin_{T \in \cT} c( T, \bY)$ where each node has at most one internal node among its children.
\end{lemma}
\begin{proof}
Consider an optimal tree $T \in \argmin_{T'\in \cT} c(T', \bY)$.
Let $v$ be a node with two internal nodes $v_1$ and $v_2$ among its children. Let $v_3,\ldots,v_k$ be the remaining children of $v$ (leaves and internal nodes). W.l.o.g. assume that $\|\cbz_{v_1}\|_1 \leq\|\cbz_{v_2}\|_1$. We construct a tree $T'$ such that $c(T', \bY) \leq c(T, \bY)$, and $T'$ has fewer nodes with two inner nodes among their children. Repeatedly applying this procedure we will get an optimal tree without nodes with more than one inner node among its children.

We define the tree $T'$ as $T$ where the node $v_1$ is no longer a child of $v$ but rather a child of $v_2$. It is easy to see that the only two nodes which change their costs after this transformation are $v$ and $v_2$. The cost of $v$ is decreased by $\|\cbz_v \|_1$ (because the $\deg_v$ is decreased by $1$) while the cost of $v_2$ is increased by $\|\cbz_{v_{2}}\|_1$ (because the $\deg_{v_2}$ is increased by $1$). Note that $\|\cbz_{v_{2}}\|_1\leq \|\cbz_{v}\|_1$ (since $v$ is the parent of $v_2$ in $T$). This completes the proof.

\end{proof}

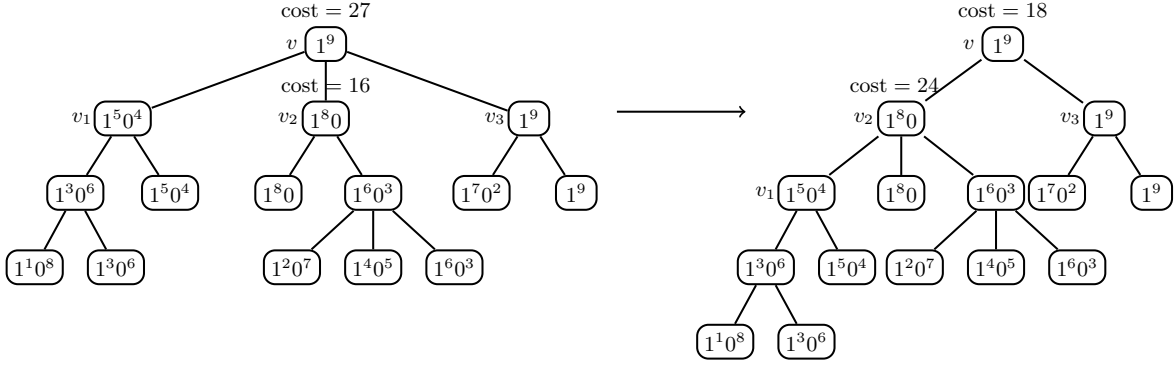
\begin{figure}[ht]
\scalebox{0.9}{
\begin{tikzpicture}%
[every node/.style={draw=black,rectangle,minimum width=6mm,minimum height=5mm,black,inner sep=.3mm,rounded corners=1ex,draw,font=\footnotesize,inner sep=3pt},
level 1/.style={sibling distance=30mm},
level 2/.style={sibling distance=14mm}, 
level 3/.style={sibling distance=12mm}, 
level distance=11mm,line width=.3mm,
scale=1,transform shape]

\node at (5,6) {$1^9$} [grow=down]
  child {node {$1^50^4$}
    child {node {$1^30^6$}  
    child {node {$1^10^8$}} 
    child {node {$1^30^6$}}
    } 
    child {node {$1^50^4$}  }   
  }
  child {node {$1^80$}
  child {node {$1^80$}}
  child {node {$1^60^3$}
  child {node {$1^20^7$}}
  child {node {$1^40^5$}}
  child {node {$1^60^3$}}
  }
  }
  child {node {$1^9$}
  child {node {$1^70^2$}}
  child {node {$1^9$}}
  }
  ;
  
  \node at (15,6) {$1^9$} [grow=down]
  
  child {node {$1^80$}
  child {node {$1^50^4$}
    child {node {$1^30^6$}  
    child {node {$1^10^8$}} 
    child {node {$1^30^6$}}
    } 
    child {node {$1^50^4$}  }   
  }
  child {node {$1^80$}}
  child {node {$1^60^3$}
  child {node {$1^20^7$}}
  child {node {$1^40^5$}}
  child {node {$1^60^3$}}
  }
  }
  child {node {$1^9$}
  child {node {$1^70^2$}}
  child {node {$1^9$}}
  }
  ;

\draw[->] (9.3,5) -- (11.2, 5);

\node[draw=none] at (5,6.5) {$\text{cost}=27$};  
\node[draw=none] at (5,5.4) {$\text{cost}=16$};  

\node[draw=none] at (4.5,6) {$v$};  
\node[draw=none] at (1.4,4.9) {$v_1$};  
\node[draw=none] at (4.45,4.9) {$v_2$};  
\node[draw=none] at (7.5,4.9) {$v_3$};

\node[draw=none] at (15,6.5) {$\text{cost}=18$};  
\node[draw=none] at (13.4,5.4) {$\text{cost}=24$};  

\node[draw=none] at (14.5,6) {$v$};  
\node[draw=none] at (11.5,3.8) {$v_1$};
\node[draw=none] at (12.95,4.9) {$v_2$}; 
\node[draw=none] at (16,4.9) {$v_3$}; 
\end{tikzpicture}
}
\caption{An~example of the transformation from Lemma~\ref{lem:innerdegree}. The total cost of the two nodes $v$ and $v_2$ which change their costs under this transformation is reduced from $43$ to $42$.}
\end{figure}

\begin{lemma}
\label{lem:matryoshka}
Let $\bY$ be an instance of the nested multi-label case. There exists a $k\ge 1$, indices $1=i_1<\ldots<i_{k}=m$, and an optimal PLT tree $T^*$ which has the following form:
\begin{itemize}
\item the internal nodes are denoted by $v_1,\ldots, v_{k-1}$, the root is $v_{k-1}$, and the leaves are $\cby_1,\ldots,\cby_m\; ;$
\item for $k>j\geq 2$, the internal node $v_j$ has children $v_{j-1}$ and $\cby_{i_j +1},\ldots,\cby_{i_{j+1}} \; ;$
\item the node $v_1$ has children $\cby_{i_1},\ldots,\cby_{i_2}\; .$
\end{itemize}
\end{lemma}
\begin{proof}
For $\bY = [\cby_1, \dots, \cby_m]$ such that $\cby_1\leq \ldots\leq \cby_m$, by Lemma~\ref{lem:innerdegree}, there exists an optimal tree $T^*$ where each internal node has at most one internal node among its children. Therefore, there exists some $k$, indices $1=i_1<\ldots<i_k=m$, and a permutation $\pi$ such that $T^*$ is as follows:
\begin{itemize}
\item the internal nodes are denoted by $v_1,\ldots, v_{k-1}$, the root is $v_{k-1}$, and the leaves are $\cby_1,\ldots,\cby_m\; ;$
\item for $k>j\geq 2$, the internal node $v_j$ has children $v_{j-1}$ and $\cby_{\pi(i_j +1)},\ldots,\cby_{\pi(i_{j+1})} \; ;$
\item the node $v_1$ has children $\cby_{i_1},\ldots,\cby_{i_2}\; .$
\end{itemize}
Thus, it suffices to show that the identity permutation $\pi=\id$ minimizes the cost of the tree $c(T^*,\bY)$. If $\pi=\id$, then the cost of the internal node $v_j$ is $c(v_j)_\id=\deg_{v_j}\|\cbz_{v_{i_{j+1}}}\|_1 =  (i_{j+1}- i_j+1)\|\cby_{i_{j+1}}\|_1$. Note that this holds for $v_1$ as well, since its children are $\cby_1,\ldots,\cby_{i_2}$. Now, assume that there exists a permutation $\pi'$ for which we get a smaller cost of the tree. Then, the cost $c(v_j)_{\pi'}$ of at least one internal node $v_j$ is smaller under the permutation $\pi'$: $c(v_j)_{\pi'}<c(v_j)_{\id}$ . Note that 
\[
c(v_j)_{\pi'}=\deg_{v_j}\|\cbz_{v_{i_{j+1}}}\|_1= (i_{j+1}- i_j+1)\max_{1\le k \le i_{j+1}} \|\cby_{\pi'(k)}\|_1\geq(i_{j+1}- i_j+1)\|\cby_{i_{j+1}}\|_1=c(v_j)_\id
\]
which finishes the proof.
\end{proof}
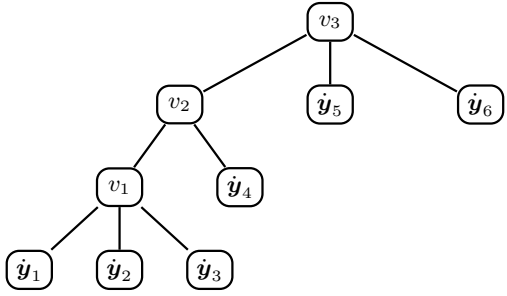
\begin{figure}[ht]
\begin{center}
\scalebox{1}{
\begin{tikzpicture}
[every node/.style={draw=black,rectangle,minimum width=6mm,minimum height=5mm,black,inner sep=.5mm,rounded corners=1ex,draw,font=\footnotesize,inner sep=3pt},
level 1/.style={sibling distance=20mm},
level 2/.style={sibling distance=16mm}, 
level 3/.style={sibling distance=12mm}, 
level distance=11mm,line width=.3mm,
scale=1,transform shape]

\node at (5,6) {$v_3$} [grow=down]
  child {node {$v_2$}
    child {node {$v_1$}  
    child {node {$\cby_1$}} 
    child {node {$\cby_2$}} 
    child {node {$\cby_3$}}
    } 
    child {node {$\cby_4$}  }   
  }
  child {node {$\cby_5$}
  }
  child {node {$\cby_6$}
  }
  ;
 
\end{tikzpicture}
}
\caption{An~example of an optimal solution from Lemma~\ref{lem:matryoshka}.}
\end{center}
\end{figure}

Note that by the nested property, an inner node $v_j$ for $j\geq 2$ has cost $c(v_j)=\deg_{v_j}\|\cbz_{v_j}\|_1 =  (i_{j+1}- i_j+1) \|\cby_{i_{j+1}}\|_1$. Thus, the problem of finding an optimal solution of the form guaranteed by Lemma~\ref{lem:matryoshka} can be stated as follows: Given a sequence of $m$ non-negative integers $A=(\|\cby_{1}\|_1,\ldots,\|\cby_{m}\|_1)$ such that $\|\cby_{1}\|_1\leq\ldots\leq \|\cby_{m}\|_1$, we need to partition $A$ into $k$ contiguous subsequences $S_1,\ldots,S_k$, in order to minimize the value of $|S_1|\cdot \max(S_1)+\sum_{j=2}^k (|S_j|+1) \cdot \max(S_j)$. 

Given the structural result of Lemma~\ref{lem:matryoshka}, one can find an optimal solution for the nested multi-label case by a simple dynamic programming algorithm in quadratic time $(m^2+mn)$. Instead, we will show that this problem can actually be solved in linear time $O(m)$, assuming a sparse representation of the input, for example as a sequence of the numbers $\|\cby_1\|_1\leq \ldots\leq \|\cby_m\|_1$. For this, we will reduce the problem to the concave least-weight sequence problem in time $O(m)$, and the latter problem can be solved in linear time $O(m)$~\citep{Wilber_1998}.
\begin{definition}[Concave least-weight sequence]
Let $n$ be an integer, and $w(i,j)$ be a real-valued function defined for integers $0\leq i, j \leq n$ with the property that $w(i_0,j_0)+w(i_1,j_1)\leq w(i_0,j_1)+w(i_1,j_0)$ for all $0\leq i_0<i_1<j_0<j_1\leq n$. The concave least-weight sequence problem is to find an integer $k\geq 1$ and a sequence of integers $0=\ell_0<\ell_1<\ldots<\ell_{k-1}<\ell_k=n$ such that $\sum_{j=0}^{k-1}w(l_i,l_{i+1})$ is minimized.
\end{definition}

Now we are ready to present the main result of this section. 
\begin{theorem}\label{thm:matryoshka}
There exists an $O(m)$ time algorithm that solves the nested multi-label \bobby{} problem exactly.
\end{theorem}
\begin{proof}%
Recall that it suffices to reduce the problem of partitioning $A=(a_1,\ldots,a_m)$ with $a_1\leq\ldots\leq a_m$ into $k$ contiguous sequences $S_1,\ldots,S_k$ minimizing $|S_1|\cdot \max(S_1)+\sum_{j=2}^k (|S_j|+1) \cdot \max(S_j)$ to the concave least-weight sequence problem.

We will define the function $w(\cdot,\cdot)$ of the concave least-weight sequence problem such that for $i<j$, $w(i,j)$ corresponds to taking the set $S=\{a_{i+1},\ldots, a_{j}\}$. Formally,
\[
    w(i,j)= 
\begin{cases}
    (j-i+1)a_j,& \text{if } 0<i<j\,;\\
    ja_j,              & \text{if } 0=i<j\,.
\end{cases}
\]

It now remains to show that the function $w(\cdot,\cdot)$ is concave: for all $0\leq i_0<i_1<j_0<j_1\leq n$, $w(i_0,j_0)+w(i_1,j_1)\leq w(i_0,j_1)+w(i_1,j_0)$. For this, consider the following two cases.

Case 1: $i_0>0$.
\begin{align*}
w(i_0,j_0)&+w(i_1,j_1) - w(i_0,j_1) - w(i_1,j_0) \\
&= (j_0-i_0+1)a_{j_0} +(j_1-i_1+1)a_{j_1} - (j_1-i_0+1)a_{j_1} - (j_0-i_1+1)a_{j_0}\\
&=(i_1-i_0)a_{j_0} - (i_1-i_0)a_{j_1}\\
&=(i_1-i_0)(a_{j_0}-a_{j_1})\\
&\leq 0\,.
\end{align*}

Case 2: $i_0=0$.
\begin{align*}
w(i_0,j_0)&+w(i_1,j_1) - w(i_0,j_1) - w(i_1,j_0) \\
&= (j_0-i_0)a_{j_0} +(j_1-i_1+1)a_{j_1} - (j_1-i_0)a_{j_1} - (j_0-i_1+1)a_{j_0}\\
&= (1-i_1)a_{j_1}-(1-i_1)a_{j_0}\\
&= (1-i_1)(a_{j_1}-a_{j_0})\\
&\leq 0\,.
\end{align*}

\end{proof}

\section{Prediction complexity} %
\label{sec:prediction}

We consider a prediction for a feature vector $\bx$ in which we find all labels such that:
$$
\heta_j(\bx) \ge \tau \,, \quad j \in \calL \,,  
$$
where $\tau \in [0,1]$ is a threshold. A natural choice of $\tau$ is $0.5$ as it leads to the optimal predictions for the Hamming loss~\citep{Dembczynski_et_al_2012a}. The threshold-based prediction can also be used for maximizing the micro- and macro-F measures as shown in~\citep{Kotlowski_Dembczynski_2015,Koyejo_et_al_2015}. Consider the tree search procedure presented in Algorithm~\ref{alg:tbs_prediction}. It starts with the root node and traverses the tree by visiting the nodes $v \in V_T$ for which 
$\heta_{pa(v)}(\bx)\geq \tau$. Obviously, the prediction consists of labels corresponding to the visited nodes.

Let us define formally the \emph{prediction cost}~$\predcost_{\tau}(T,\bx)$ for a single instance~$\bx$ as the number of calls to node classifiers in the Algorithm~\ref{alg:tbs_prediction} during prediction, i.e.:\footnote{Notice that the time complexity of Algorithm~\ref{alg:tbs_prediction} is $O(\predcost_{\tau}(T,\bx))$.}
\[
\predcost_{\tau}(T,\bx)= 1 + \sum\limits_{v\in V_T}\assert{\heta_{pa(v)}(\bx)\geq \tau} = 1 + \sum\limits_{v\in V_T}\assert{\heta_{v}(\bx)\geq \tau} \cdot \deg_v.
\]
The \emph{expected prediction cost} is then $C_{\prob(\bx), \tau}(T) = \mathbb{E}_{\bx} [ \predcost_{\tau}(T,\bx)]$.

\begin{algorithm}[ht]
\caption{\Algo{PLT.Predict}$(T, \bx, \tau)$}
\label{alg:tbs_prediction}
\begin{small}
\begin{algorithmic}[1] 
\State $\hat\by = \vec{0}$, $\calQ = \emptyset$ \Comment{Initialize the prediction vector to all zeros and a stack}
\State $\calQ\mathrm{.add}((r_T, \heta(\bx, r_T)))$ \Comment{Add the tree root and the corresponding estimate of probability}
\While{$\calQ \neq \emptyset$}  \Comment{In the loop}
	\State $(v, \heta_v(\bx)) = \calQ\mathrm{.pop}()$ \Comment{Pop an element from the stack}
	\If{$\heta_v(\bx) \ge \tau$} \Comment{If the probability estimate is greater or equal $\tau$}
	\If{$v$ is a leaf}  \Comment{If the node is a leaf}
		\State $y_v = 1$ \Comment{Set the corresponding label in the prediction vector}
	\Else \Comment{If the node is an internal node}
	\For{$v' \in \ch{v}$} \Comment{For all child nodes}
	    \State $\heta_{v'}(\bx) = \heta_v(\bx) \times \heta(\bx,v')$ \Comment{Compute $\heta_{v'}(\bx)$}
		\State $\calQ\mathrm{.add}((v', \heta_{v'}(\bx)))$  \Comment{Add the node and the computed probability estimate}
	\EndFor	
	\EndIf 
	\EndIf 
\EndWhile
\State \textbf{return} $\hat\by$ \Comment{Return the prediction vector}
\end{algorithmic}
\end{small}
\end{algorithm}

To express the complexity of the prediction algorithm for a given $\bx$ in terms of the depth and the degree of $T$, we assume that $\sum_{i=1}^m \eta_j(\bx)$ is upperbounded by a constant $P$ (e.g., for the case of $k$-sparse multi-label classification we have $P = k$). Let us denote the $L_1$-estimation error in each node $v \in V_T$  by $\epsilon_v$, i.e., $\epsilon_v = |\eta(\bx, v) -  \heta(\bx, v)|$. Then, from Theorem~\ref{thm:estimation_regret}, we have that  $|\eta_v(\bx) -  \heta_v(\bx)| \le \sum_{v' \in \Path{v}} \eta_{\pa{v}}(\bx) \cdot \epsilon_v$. Assuming that $\heta_v(\bx)$ are properly normalized to satisfy Proposition~\ref{prop:node_cond_prob_interval}, we prove the following result which is a counterpart of Proposition~\ref{prop:cost_upperbound} for training cost.
\begin{theorem}
\label{thm:comp_tresh_pred}
For Algorithm~\ref{alg:tbs_prediction} with threshold $\tau$ and any $\bx \in \calX$, we have that:
\begin{equation}
\predcost_{\tau}(T,\bx) \le  1 + \lfloor \hat{P}/\tau \rfloor \cdot \depth_T \cdot \deg_T \,,    
\label{eq:comp_tresh_pred}
\end{equation}
where $\hat{P} = \sum_{j=1}^m \heta_j(\bx) \leq P + \sum_{v \in V_T} |L(v)| \eta_{\pa{v}}(\bx) \cdot \epsilon_v$, $\depth_T = \max_{v \in L_T} \lenpath_v - 1$, and $\deg_T = \max_{v \in V_T} \deg_v$.
%
\end{theorem}
\begin{proof}
Let us first show an upper bound on $\hat{P}$:
\begin{eqnarray}
\hat{P} = \sum_{j=1}^m \heta_j(\bx) & \leq & \sum_{j=1}^m \left(\eta_j(\bx) + \sum_{v \in \Path{\ell_j}}  \eta_{\pa{v}}(\bx) \cdot \epsilon_v \right) \label{eqn:phat_bound}\\
 & \le & P + \sum_{j=1}^{m} \sum_{v \in \Path{\ell_j}} \eta_{\pa{v}}(\bx) \cdot \epsilon_v  \nonumber \\
 & = & P +  \sum_{v \in V_T}  |L(v)| \eta_{\pa{v}}(\bx) \cdot \epsilon_v \nonumber \,,
\end{eqnarray}
where (\ref{eqn:phat_bound}) follows from Theorem~\ref{thm:estimation_regret} and $P$ is an upper bound on $\sum_{j=1}^m \eta_j(\bx)$.


As stated before the theorem, we assume that the estimates satisfy the following:
\begin{equation}
\heta_v(\bx) \le \min \left \{1, \sum_{v' \in \childs{v}} \heta_{v'}(\bx)  \right \} \,,
\label{eqn:node_heta_prob_ub}
\end{equation}
and
\begin{equation}
\max \left \{\heta_{v'}(\bx), v' \in \childs{v} \right \} \le \heta_v(\bx)  \,.
\label{eqn:node_heta_prob_lb}
\end{equation}
These are important properties of the true probabilities stated in Proposition~\ref{prop:node_cond_prob_interval}. To satisfy them by the estimates we can perform the following normalization steps for the child nodes during prediction:
\begin{eqnarray*}
\heta_{v'}(\bx) & \leftarrow & \min(\heta_{v'}(\bx), \heta_{v})\,, \quad \textrm{for all~} v' \in \childs{v} \,, \\
\heta_{v'}(\bx) & \leftarrow & \frac{ \heta_{v'}(\bx) \cdot \heta_{v}}{\sum_{v' \in \childs{v}} \heta_{v'}(\bx)}\,, \quad \textrm{if~} \heta_{v} > \sum_{v' \in \childs{v}} \heta_{v'}(\bx) \,, \quad \textrm{for all~} v' \in \childs{v} \,.
\end{eqnarray*}
The error terms $\epsilon_v$ concern then the normalized estimates.   

Now, we move to the main part of the proof. Consider the subtree $T'$ of $T$, which consists of all nodes $v \in T$ for which $\heta_v(\bx) \ge \tau$. If there are no such nodes, from the pseudocode of Algorithm~\ref{alg:tbs_prediction}, we see that only the root classifier is called. The upperbound~(\ref{eq:comp_tresh_pred}) in this case obviously holds. However, it might not be tight as $\heta_v(\bx) < \tau$ does not imply $\hat P \ge \tau$ because of (\ref{eqn:node_heta_prob_ub}). 

If $T'$ has at least one node, Algorithm~\ref{alg:tbs_prediction} visits each node of $T'$ (i.e., calls a corresponding classifier and add the node to a stack), since for each parent node we have (\ref{eqn:node_heta_prob_lb}). Moreover, Algorithm~\ref{alg:tbs_prediction} visits all children of nodes $T'$ (some of them are already in $T'$). Let the subtree $T''$ consist of all nodes of $T'$ and their child nodes. Certainly $T' \subseteq T'' \subseteq T$. To prove the theorem we count first the number of nodes in $T'$ and then the number of nodes in $T''$, which gives as the final result. 

If the number of nodes in $T'$ is greater than or equal to 1, then certainly $r_{T}$ is in $T'$. Let us consider next the number of leaves of $T'$. Observe that $\sum_{v \in L_{T'}} \heta_v(\bx) \le \hat P$. This is because $\sum_{v \in L_{T'}} \heta_v(\bx) \le \sum_{v \in L_{T}} \heta_v(\bx) \le \hat P$  , i.e., $v \in L_{T'}$ might be an internal node in $T$ and its $\heta_v(\bx)$ is at most the sum of probability estimates of the leaves underneath $v$ according to (\ref{eqn:node_heta_prob_ub}).  From this we get the following upper bound on the number of leaves in $T'$:
\begin{equation}
|L_{T'}| \le \lfloor \hat P/\tau \rfloor \,.
\label{eqn:num_leaves_ub}
\end{equation}
Since the degree of internal nodes in $T'$ might be $1$, to upperbound the number of all nodes in $T'$ we count the number of nodes on all paths from leaves to the root, but counting the root node only once, i.e.:
$$
|V_{T'}| \le 1 + \sum_{v \in L_{T'}} (\lenpath_v - 1) \,.
$$
Next, notice that for each $v \in T'$ its all siblings are in $T''$ unless $v$ is the root node. This is because if non-root node $v$ is in $T'$ then its parent is also in $T'$ according to (\ref{eqn:node_heta_prob_lb}) and $T''$ contains all child nodes of nodes in $T'$.  The rest of nodes in $T''$ are the child nodes of leaves of $T'$, unless a leaf of $T'$ is also a leaf of $T$. Therefore, we have 
$$
|V_{T''}| \le 1 +  \sum_{v \in L_{T'}} \deg_T (\lenpath_v - 1) + \sum_{v \in L_{T'}} \deg_T \assert{v \not \in L_{T}} \,,
$$
with $\deg_T$ being the highest possible degree of a node.
Since (\ref{eqn:num_leaves_ub}) and
$$
\lenpath_v - 1 + \assert{v \not \in L_{T}} \le \depth_T \,,
$$ 
i.e., the longest path cannot be longer than the depth of the tree plus 1, we finally get:
$$
|V_{T''}| \le  1 +  \lfloor \hat P/\tau \rfloor \cdot \depth_T \cdot \deg_T \,. 
$$
This ends the proof as the number of nodes in $T''$ is equivalent to the number of calls to the node classifiers, i.e., $\predcost_{\tau}(T,\bx)$.

\end{proof}
\begin{remark}
For a tree of constant $\deg_T=\lambda (\ge 2)$ and $\depth_T=\log_\lambda{m}$, the cost of Algorithm~\ref{alg:tbs_prediction} is $O(\log m)$ if $\hat{P}$ is upper bounded by a constant or node classifiers predict with no error, i.e., $\epsilon_v = 0$, for all $v \in V_T$.  
\end{remark}

The above result does not, however, relate directly the prediction cost to the training cost. The next theorem shows this relation in terms of expected costs. 
\begin{theorem}
\label{thm:inferenceCostUpperBound}
Using the notation above, it holds that
\[
C_{\prob(\bx), \tau}(T) \leq \frac{1}{\tau}\bigg(C_\prob(T)+\sum\limits_{v\in V_T}\mathbb{E}_{\bx}\left[\eta_{\pa{v}}(\bx)\cdot  \epsilon_v \right]\cdot|T(v)|\cdot\deg_v\bigg)-\frac{1-\tau}{\tau},
\]
where~$|T(v)|$ denotes the number of inner nodes in the subtree~$T(v)$ rooted at~$v$.
\end{theorem}
\begin{proof}
We can upper bound the expected inference cost as follows:
\begin{align}
\mathbb{E}_{\bx}\left[\predcost_{\tau}(T,\bx)\right] & 
  =\mathbb{E}_{\bx}\left[1+\sum\limits_{v\in V_T}\mathbb{I} \left\{ \heta_{v}(\bx)\geq \tau\right\}\deg_v\right] \leq
  \mathbb{E}_{\bx}\left[1+\sum\limits_{v\in V_T}\frac{\heta_{v}(\bx)}{\tau}\deg_v\right] \notag \\
  & \leq
  1+\sum\limits_{v\in V_T}\frac{\mathbb{E}_{\bx}\left[\eta_{v}(\bx)+|\heta_{v}(\bx)-\eta_{v}(\bx)| \right]}{\tau}\deg_v \notag \\
  & \leq
  \frac{1}{\tau}\left(1+\sum\limits_{v\in V_T}\mathbb{E}_{\bx}\bigg[\eta_v(\bx)+\sum\limits_{v'\in\text{Path}(v)}\eta_{\pa{v'}}(\bx)\cdot \epsilon_v \bigg]\cdot \deg_v\right)-\frac{1-\tau}{\tau} \label{eq:applyTh1} \\
  & \leq
  \frac{1}{\tau}\bigg(C_\prob(T)+\sum\limits_{v\in V_T}\mathbb{E}_{\bx}\left[\eta_{\pa{v}}(\bx)\cdot \epsilon_v \right]\cdot|T(v)|\cdot\deg_v\bigg)-\frac{1-\tau}{\tau},
\end{align}
where (\ref{eq:applyTh1}) follows from Theorem~\ref{thm:estimation_regret}.
\end{proof}

\begin{proposition}\label{prop:InfeerenceCostTightExample}
The bound in Theorem~\ref{thm:inferenceCostUpperBound} is tight.
\end{proposition}
\begin{proof}
We construct a tight example as follows. First, we assume that the classifier predicts with no error,~$\epsilon_v=0$ for all~$\bx,v\in V_T$. Second, we consider all node probabilities~$\eta(\bx,v)=1$ for all~$v\in V_T$ except for the root node, for which~$\eta(\bx,r)=\tau$. In this case all conditional probabilities~$\eta_j(\bx)=\tau$, $j \in \calL$, and we can write the exact values of the expected costs:
\[
C_{\prob}(T)=1+\sum\limits_{v\in V_T} \tau\deg_v=\tau\cdot|V_T|+1,
\]
and
\[
C_{\prob(\bx), \tau}(T)=1+\sum\limits_{v\in V_T} \deg_v=|V_T|+1.
\]
\end{proof}

\begin{corollary}
In the case of exact predictions,~$\epsilon_v=0$ for all~$\bx,v\in V_T$, for $\tau=0.5$ we have
\[
C_{\prob(\bx), 0.5}(T)\leq 2C_\prob(T)-1.
\]
\end{corollary}

\begin{remark}
A natural question is whether the lower bound on~$C_{\prob(\bx), \tau}(T)$ exists of the form
\[
C_{\prob(\bx), \tau}(T)=\Omega\left(C_\prob(T)\right)
\]
under the assumption~$\epsilon_v=0$ for all~$\bx,v\in V_T$.
To see that no such bound exists without additional assumptions on~$\eta_v(\bx)$, one can consider the following example: let all~$\eta_v(\bx)=\tau-\varepsilon$, $v \in V_T$, for $\varepsilon > 0$. In this case we have
\[
C_\prob(T)=(\tau-\varepsilon)\cdot|V_T|+1,\quad C_{\prob(\bx), \tau}(T)=1.
\]
\end{remark}

\begin{remark}
The above theorems contain the term $\eta_{\pa{v}}(\bx) \cdot \epsilon_v $ either multiplied by $|L(v)|$ or $|T(v)|$. It nicely shows an interplay between the $L_1$-estimation error ($\epsilon_v$) and the \emph{importance} ($\eta_{\pa{v}}(\bx)$ multiplied by $|L(v)|$ or $|T(v)|$) of node $v \in V_T$. For example, the root node has the highest importance, but the error there should be the smallest as all training examples are used there and the learning problem is relatively simple as the root classifier has to estimate the probability whether there exists at least one label in $\by$. In many cases, we can assume that this probability is 1 and the error will be 0 then. In turn, the error in the leaves nodes can be substantial as the number of training examples there is the smallest, but their importance is also the smallest. 
\end{remark}

\section{Conclusions and future work}
\label{sec:conclusion}

In this paper, we addressed the problem of optimizing the training and test costs of \Algo{PLT}s. We showed that optimizing the training cost is an NP-complete problem, nevertheless it has several tractable special cases for which either exact or approximate solution can be efficiently found. 
We also show guarantees for the test cost and characterize its relation to the training cost. 

Several exciting open questions have arisen from this work. One is to prove either a reasonable lower bound for the general multi-label case or to prove the tightness of the $O(\log m)$ approximation. Second is to find a tree with the optimal statistical error and show its relation to the computational cost objective.

\bibliographystyle{abbrvnat}

\appendix

\section{Proof of Theorem~\ref{thm:np_complete}}
\npcomplete*
\begin{proof}%
It is easy to see that $\bobby\in\text{NP}$. Indeed, given an optimal \bobby{} tree, one can in polynomial time verify that its cost is at most $w$.\footnote{We also need to verify that there exists an optimal tree of polynomial size. Since there always exists an optimal tree where each internal node has at least $2$ children, there exists an optimal tree with at most $m-1$ internal nodes.}

Now we will show that \bobby{} is NP-hard by giving a polynomial time reduction from Clique to \bobby{}. The reduction consists of two steps. In the first step, we reduce Clique to Directed Clique with $k=\frac{2|V|}{3}$. In the second step, we reduce this version of the Directed Clique problem to \bobby{}.

\paragraph{Preparations.}
Let $G'=(V',E')$ and $1\leq k' \leq |V'|$ be an instance of the Clique problem. W.l.o.g. let us assume that $G'$ does not have isolated nodes. We will now construct an undirected graph $G^*=(V^*,E^*)$ and $k^*=\frac{2|V^*|}{3}$ such that $G'$ contains a clique of size $k'$ if and only if $G^*$ contains a cliques of size $k^*$.

First, we add all nodes and edges of $G'$ to $G^*$. If $k'\leq\frac{2|V'|}{3}$, then we add $2|V'|-3k'$ nodes to $G^*$ which are connected to all nodes of $G^*$. It is now easy to see that $G'$ contains a $k'$-cliques if and only if $G^*$ contains a $k^*=\frac{2|V^*|}{3}$-clique. 

If $k'>\frac{2|V'|}{3}$, then we first make $k'$ even. Namely, if $k'$ is odd, we increase $k'$ by one and add one node to $G'$ connected to all other nodes. Now we add a path on $\frac{3k-2|V'|}{2}$ nodes to $G^*$, and connect one node of the path to an arbitrary node of $G$. Again, it is easy to see that the new instance of the clique problem (with $k^*=\frac{2|V^*|}{3}$) is equivalent to the original one. We note that the constructed graph does not have isolated nodes.

We now construct a \label{directed} graph $G=(V, E)$ and $k=\frac{2|V|}{3}$ such that it contains a directed $k$-clique (with self-loops) if and only if $G'$ has a $k'$-clique. To this end, we set $V=V^*$, for each edge $\{u,v\}$ we create two arcs $(u,v)$ and $(v,u)$ in $G$, and we also add all $|V|$ self-loops in $G$.

\paragraph{Reduction.}
Let $n=|V|$ and $m=|E|$, w.l.o.g. assume that $n\geq 30$ and $n$ is a multiple of $30$. Now we reduce the Clique problem with $k=\frac{2n}{3}$ to an instance of $\bobby$ with $m$ vectors in $d=n+\ell$ dimension where 
\begin{align*}
\ell=0.145n^3 \; .
\end{align*}
(Since $n$ is a multiple of $10$, $\ell$ is an integer.)
For every arc $e=(v_i,v_j)$ of the graph $G$, we create a Boolean vector $\cby_e\in\{0,1\}^d$ as follows. In the first $n$ coordinates, we set only the $i$th and $j$th coordinates to $1$ (if $e$ is a self-loop, we only set one coordinate to $1$). The last $\ell$ coordinates are always set to $1$. Now we set
\begin{align*}
w=\ell(m+1)+k^3+nm-nk^2+n=\ell(m+1)-\frac{4n^3}{27}+nm+n \; .
\end{align*}
We claim that this instance of the \bobby{} problem has a $T\in \cT$ tree of cost at most $w+n$ if and only if $G$ contains a $k$-clique. We also remark that both reductions run in time polynomial in $n$ and $m$.

\paragraph{Correctness.}
First we show that if $G$ contains a $k$-clique, then there is a $T\in \cT$ of cost at most $w+n$. We construct a tree with two internal nodes (including the root) as follows. All the $k^2$ arcs of the directed $k$-clique (with self-loops) are connected to the internal node $v$. This internal node and the remaining $m-k^2$ arcs are connected to the root $r$. Let us now compute the labels $\cbz_v$ and $\cbz_r$. Since all children of $v$ correspond to the edges forming a $k$-clique, the vector $\cbz_v$ has exactly $k$ ones in the first $n$ coordinates, and it has all $\ell$ ones in the remaining coordinates. Since there are no isolated nodes in $G$, the vector $\cbz_r$ is $1^d$. Now the cost $c(v) =(k+\ell)k^2$, $c(r)=(n+\ell)(m-k^2+1)$, and the total cost of the constructed tree is
\begin{align*}
c(T, \bY)=n+c(v)+c(r)=n+(k+\ell)k^2+(n+\ell)(m-k^2+1)=n+w \; .
\end{align*}
\begin{figure}[ht]
\begin{center}
\scalebox{1}{
\begin{tikzpicture}
[every node/.style={draw=black,rectangle,minimum width=2cm,minimum height=5mm,black,inner sep=.5mm,rounded corners=1ex,draw,font=\footnotesize,inner sep=3pt},
level 1/.style={sibling distance=36mm},
level 2/.style={sibling distance=24mm}, 
level distance=12mm,line width=.3mm, scale=1,transform shape]

\draw [decorate,decoration={brace,amplitude=10pt,mirror},xshift=0pt,yshift=-4pt]
(-3.8,3.4) -- (3.1,3.4) node [draw=none,black,midway,yshift=-6mm] 
{$k^2$ leaves};

\draw [decorate,decoration={brace,amplitude=10pt,mirror},xshift=0pt,yshift=-4pt]
(2,4.6) -- (11.7,4.6) node [draw=none,black,midway,yshift=-6mm] 
{$m-k^2$ leaves};
\node[draw=none] at (-0.3,5.3) {node $v$};
\node[draw=none] at (5,6.5) {root $r$};

\node at (5,6) {$1^{n+\ell}$} [grow=down]
  child {node {$1^k0^{n-k}1^\ell$}
    child {node {$10^{n-1}1^\ell$}  } 
    child {node {$\ldots$}  }   
    child {node {$0^{k-1}10^{n-k}1^\ell$}  }   
  }
  child {node {$0^k110^{n-k-2}1^\ell$}}
  child {node {$\ldots$}}
  child {node {$0^k1010^{n-k-3}1^\ell$}
  }
  ;

\end{tikzpicture}
}
\caption{The tree corresponding to a $k$-clique. All $k^2$ arcs forming the clique are connected to the internal node $v$ on the left, all the remaining arc are connected dicrectly to the root $r$. The label of the root is $\cbz_r=1^d=1^{n+\ell}$, and the label $\cbz_v$ of the node $v$ has $k$ ones in the first part (corresponding to the nodes of the clique) and $\ell$ ones in the second part.}
\end{center}
\end{figure}
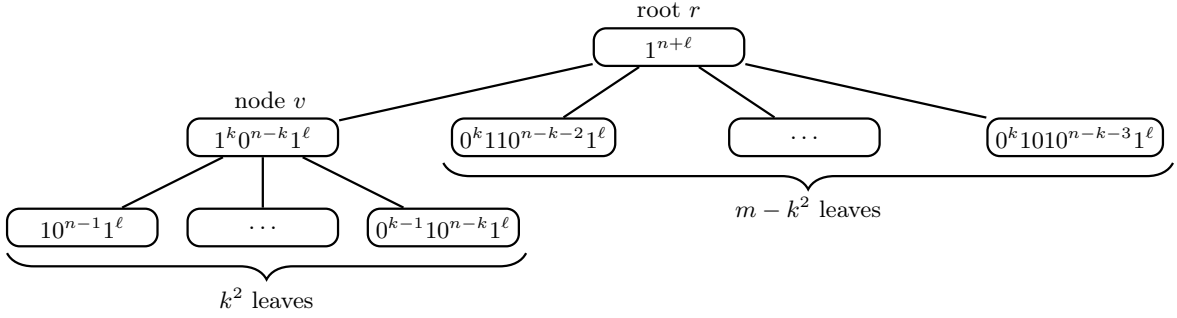

Now we will show that a tree $T$ of cost at most $w+n$ gives us a $k$-clique in $G$. To this end, we will consider three cases. Let $i$ be the number of internal nodes of $T$ (recall that the root counts as an internal node, thus, $i\geq 1$). We will show that if $i=1$ or $i>3$, then the cost of $T$ is larger than $w+n$, and if $i=2$ then any tree of cost $w+n$ must look exactly as the one above, which gives us a $k$-clique.

\begin{itemize}
\item $i=1$. If $T$ has only one internal node (the root), then the cost of the tree is the cost of the root plus $n$. Since there are no isolated nodes in $G$, the label of the root $\cbz_r$ has all ones. Thus,
\begin{align*}
c(r)=(n+\ell)m=nm+\ell(m+1)-0.125n^3>\ell(m+1)-\frac{4n^3}{27}+nm+n = w \; 
\end{align*}
for $n\geq 7$. This implies that $c(T,\bY)=n+c(r)>n+w$.
\item $i=2$. Assume that one internal node $v$ is connected to $e$ leaves which span $t\leq n$ nodes of $G$, where $e\leq t^2$, and the root $r$ is connected to this internal node $v$ and the remaining $m-e$ leaves. Both $\cbz_v$ and $\cbz_r$ have $\ell$ ones in the last $\ell$ coordinates, and these two nodes $v$ and $r$ together have $(m+1)$ children. Thus, the total contribution to the cost $c(T, \bY)$ of the last $\ell$ coordinates of $\cbz_v$ and $\cbz_r$ is $(m+1)\ell$. The root $\cbz_r$ has $n$ ones in the first $n$ coordinates, and $\cbz_v$ has $t$ ones. Therefore, the contribution of the first $n$ coordinates of $\cbz_v$ and $\cbz_r$ to $c(T, \bY)$ is $et+(m-e+1)n$. Now we have that $c(r)+c(v)=(m+1)\ell+et+(m-e+1)n$. We will show that $c(r)+c(v)\leq w$ if and only if $e=k^2$ and $t=k$ which corresponds to a $k$-clique in the original graph.
\begin{align*}
c(r)+c(v)-w=et-en+\frac{4n^3}{27}\geq t^3-t^2n+\frac{4n^3}{27} \; .
\end{align*}
By taking the derivative with respect to $t$, we see that this expession is greater than $0$ for $t\neq\frac{2n}{3}$. Thus, we have that $t=\frac{2n}{3}=k$, and $e=k^2$.
\item $i\geq 3$. Each inner node has $\ell$ ones in the last $\ell$ coordinates, and all inner nodes together have $(m+i-1)$ children. This means that the last $\ell$ coordinates of the corresponding vectors $\cbz$ contribute $(m+i-1)\ell$ to the cost $c(T,\bY)$. Let $m_1,\ldots,m_i$ be the numbers of leaves connected to each of the internal nodes, $\sum_{j=1}^i m_j=m$. Since $m_j$ arcs span at least $\sqrt{m_j}$ nodes, the contribution of the first $n$ coordinates of the corresponding vectors $\cbz$ to the cost $c(T,\bY)$ is at least $\sum_{j=1}^i m_j^{3/2}$. Thus, by H\"older's inequality,
\begin{align*}
c(T, \bY) \geq n+(m+i-1)\ell+\sum_{j=1}^i m_j^{3/2} \geq n+(m+i-1)\ell + \frac{m^{3/2}}{\sqrt{i}} \; .
\end{align*}
In order to show that $c(T, \bY)-n>w$ for all $i\geq 3$ and $1\leq m\leq n^2$, we show that the following function is always positive:
\begin{align*}
f(m,i)=(m+i-1)\ell+ \frac{m^{3/2}}{\sqrt{i}}-w= (i-2)\cdot 0.145n^3+\frac{m^{3/2}}{\sqrt{i}}+\frac{4n^3}{27}-nm-n\; .
\end{align*}
By taking the derivative w.r.t. $m$, we see that for every value of $i\geq 3, f(m,i)$ takes its minimum at $m=n^2$. By plugging in $m=n^2$, we get
\begin{align*}
g(i)=n^3\left(0.145(i-2)+\frac{1}{\sqrt{i}}-\frac{23}{27}\right)-n \;.
\end{align*}
By taking the derivative of $g(i)$, we get that $g(i)$ is minimized when $i=3$, in which case we have
\begin{align*}
c(T, \bY)-w-n\geq n^3\left(0.145+\frac{1}{\sqrt{3}}-\frac{23}{27}\right)-n>0.1n^3-n\geq0
\end{align*}
for $n\geq 4$.
\end{itemize}
\end{proof}

\end{document}